\def\year{2022}\relax

\documentclass[letterpaper]{article} 

\usepackage{latex22}  
\usepackage{times}  
\usepackage{helvet}  
\usepackage{courier}  
\usepackage[hyphens]{url}  
\usepackage{graphicx} 
\usepackage{natbib}  
\usepackage{caption} 
\DeclareCaptionStyle{ruled}{labelfont=normalfont,labelsep=colon,strut=off} 
\usepackage{multirow}
\usepackage{booktabs}
\usepackage{mathrsfs}
\usepackage{amsmath}
\usepackage{subfigure}
\usepackage{amssymb,amstext}
\usepackage{color}
\usepackage{bm}

\newtheorem{theorem}{Theorem}
\usepackage{algorithm}
\usepackage{algorithmic}

\usepackage{soul}
\usepackage{url}
\usepackage[hidelinks]{hyperref}
\usepackage[utf8]{inputenc}
\usepackage{amsmath}
\usepackage{subfigure}
\usepackage{booktabs}
\usepackage{mathrsfs}
\usepackage{amsfonts}
\usepackage{subfigure}
\usepackage{bbm}
\usepackage{booktabs}
\usepackage{amsmath,color}
\usepackage{multirow}
\usepackage{diagbox}
\newtheorem{proof}{Proof}
\usepackage[font=small,labelfont=bf]{caption}
\usepackage[normalem]{ulem}
\useunder{\uline}{\ul}{}
\usepackage{listings}

\title{Orthogonal Graph Neural Networks }
\author{

    Kai Guo \textsuperscript{\rm 1},
    Kaixiong Zhou \textsuperscript{\rm 2},
    Xia Hu \textsuperscript{\rm 2},
    Yu Li \textsuperscript{\rm 3},
    Yi Chang \textsuperscript{\rm 1},
    Xin Wang \textsuperscript{\rm 1}\thanks{Corresponding author: xinwang@jlu.edu.cn}
    \\
}
\affiliations{
    \textsuperscript{\rm 1}School of Artificial Intelligence, Jilin University, Changchun, China\\
    \textsuperscript{\rm 2}Department of Computer Science, Rice University, USA\\ 
    \textsuperscript{\rm 3}College of Computer Science and Technology, Jilin University, Changchun, China

    guokai20@mails.jlu.edu.cn, Kaixiong.Zhou@rice.edu, xia.hu@rice.edu, liyu18@mails.jlu.edu.cn, yichang@jlu.edu.cn, xinwang@jlu.edu.cn

}

\begin{document}

\maketitle

\begin{abstract}
Graph neural networks (GNNs) have received tremendous attention due to their superiority in learning node representations. These models rely on message passing and feature transformation functions to encode the structural and feature information from neighbors. However, stacking more convolutional layers significantly decreases the performance of GNNs. Most recent studies attribute this limitation to the over-smoothing issue, where node embeddings converge to indistinguishable vectors. Through a number of experimental observations, we argue that the main factor degrading the performance is the unstable forward normalization and backward gradient resulted from the improper design of the feature transformation, especially for shallow GNNs where the over-smoothing has not happened. Therefore, we propose a novel orthogonal feature transformation, named Ortho-GConv, which could generally augment the existing GNN backbones to stabilize the model training and improve the model's generalization performance. Specifically, we maintain the orthogonality of the feature transformation comprehensively from three perspectives, namely hybrid weight initialization, orthogonal transformation, and orthogonal regularization. 
By equipping the existing GNNs (e.g. GCN, JKNet, GCNII) with Ortho-GConv, we demonstrate the generality of the orthogonal feature transformation to enable stable training, and show its effectiveness for node and graph classification tasks.
\end{abstract}

\section{Introduction}
Graph neural networks (GNNs)~\cite{kipf2016semi} and their variants have been widely applied to analyze graph-structured data, such as social networks~\cite{tian2019inferring,zhou2019auto} and molecular networks~\cite{DBLP:conf/ijcai/ZhaoLHLZ21, hao2020asgn,zhou2020multi}. Based upon the input node features and graph topology, GNNs apply neighbor aggregation and feature transformation schemes to update the representation of each node recursively. While the neighbor aggregation passes neighborhood messages along the adjacency edges, the feature transformation aims to project node embedding to improve models' learning ability.
Despite their superior effectiveness, a key limitation of GNNs is that their performances would decrease significantly with layer stacking. Most of the previous studies ~\cite{chen2020measuring, li2018deeper, oono2019graph} attribute this limitation to the over-smoothing issue, which indicates that the node representations become indistinguishable due to the recursive neighbor aggregation upon the graph structure. A number of models have been recently proposed to alleviate the over-smoothing issue, including skip connection~\cite{chen2020simple,chen2020revisiting,Li_2019_ICCV} and graph augmentation~\cite{rong2019dropedge}. Their main ideas are to avoid the overwhelming amount of neighborhood information, and strengthen the specific node features of themselves at each graph convolutional layer.

In contrast to the previous extensive studies of over-smoothing for the extremely deep GNNs, we transfer the research attendance to explore the primary factor compromising the performances of shallow GNNs. Our research is motivated by an inconsistent observation in Figure~\ref{fig: analysis}.(d): the node classification accuracy of GNNs drops rapidly once the model depth is slightly enhanced (e.g., up to $8$ layers), where the over-smoothing status should be far from being reached. By simply removing the feature transformation module from GNNs, it is further observed that GNNs surprisingly perform steadily even with tens of graph convolutional layers. This motivates us to pose the following research question: does the stable feature transformation play the dominating role in affecting the model performance of shallow GNNs?

\textbf{Forward and backward signaling analysis.} To answer this question, we first systematically analyze the steadiness of feature transformation from both the forward inference and backward gradient directions. We apply two corresponding steadiness metrics: one measuring the 
\textcolor{black}{signal magnifications
}
of forward node embeddings and the other evaluating the norms of backward gradients. As shown in Figure~\ref{fig: analysis}.(a) and (b), it is empirically demonstrated that vanilla GNNs suffer from the forward embedding explosion and backward gradient vanishing. While the forward explosion greatly shifts the internal embedding distributions among layers to make the model training inefficient~\cite{ioffe2015batch}, the gradient vanishing hampers the tuning of feature transformation module to adapt the downstream tasks. Therefore, we \textcolor{black}{conclude and argue that} the vanilla feature transformation damages the steady model signaling at both the forward and backward directions, which in turn degrades the performances especially for shallow GNNs. 

\textbf{Orthogonal graph convolutions.} To overcome the unstable training, we propose orthogonal graph convolutions to ensure the orthogonality of feature transformation. The orthogonal weight \cite{DBLP:conf/iclr/TrockmanK21,wang2020orthogonal,vorontsov2017orthogonality} has been explored in convolutional and recurrent neural networks (CNNs and RNNs) to maintain the forward activation norm and to avoid the gradient vanishing, which can accelerate the training and improve adversarial robustness. To tailor to the graph data, we optimize the orthogonality of feature transformation from three perspectives: (i) a hybrid weight initialization to possess the trade-off between graph representation learning ability and the model's orthogonality; (ii) an orthogonal weight transformation to ensure the orthogonality in the forward inference; and (iii) an orthogonal regularization to constrain orthogonality during the backward update. The contributions are listed in the following:

\begin{itemize}
\item We propose two metrics to measure the steadiness of forward inference as well as backward gradient, and provide the systematic analysis to theoretically and empirically study the influences of unstable feature transformation on shallow GNNs. 

\item We propose the orthogonal graph convolutions, named  Ortho-GConv, to achieve the orthogonality of feature transformation and stabilize the forward and backward signaling in GNNs.

\item We conduct the comprehensive experiments on both the node and graph classification tasks to demonstrate the general effectiveness of our Ortho-GConv on various GNN backbones, including GCN, JKNet and GCNII. 

\end{itemize}

\section{Related work}
\textbf{GNNs.} 
~\cite{bruna2013spectral} is the first remarkable research on GNNs, which develops graph convolution based on spectral graph theory for graph application tasks. Later, a series of GNN variants are proposed by~\cite{kipf2016semi,defferrard2016convolutional,henaff2015deep,li2018adaptive,levie2018cayleynets}. Although these models achieve the better performance with two layers, when stacking more layers hampers their performance. Recently, several studies argue that stacking more layers causes the over-smoothing issue. These methods, such as APPNP ~\cite{DBLP:conf/iclr/KlicperaBG19}, JKNet~\cite{xu2018representation}, DropEdge ~\cite{rong2019dropedge} and GCNII~\cite{chen2020simple}, are proposed to solve the over-smoothing issue. However, \cite{liu2020towards} claim that the over-smoothing issue only happens when node representations propagate repeatedly for a large number of iterations, especially for sparse graphs. Therefore, a few propagation iterations in the GNN models are insufficient for over-smoothing to occur. On the contrary, we argue that unstable forward and backward signaling results in the poor performance of shallow GNN models.

\paragraph{Orthogonal initialization and transformation.} The advantages of orthogonal weight initialization, i.e., ensuring the signals propagate through deep networks and preserving gradient norms, are explained by~\cite{DBLP:journals/corr/SaxeMG13}. Recently, there have been several studies exploring orthogonal initialization in CNNs. \cite{xiao2018dynamical} propose a method for orthogonal convolutions and demonstrated, which allows the CNN model to effectively explore large receptive fields without batch normalization or residual connection.
\cite{DBLP:conf/iclr/TrockmanK21} propose the Calay transformation to constrain the parameters in convolutional layers to make them orthogonal.
\cite{DBLP:conf/cvpr/Huang00WYL020} propose an efficient and stable orthogonalization method to learn a layer-wise orthogonal weight matrix in CNNs. Furthermore, the gradient norm preserving properties can also benefit from remembering long-term dependencies in RNNs. 
\cite{vorontsov2017orthogonality} propose the constrained transformation matrices to make orthogonal and address the gradient vanishing and exploding in RNNs.

\begin{figure*}
    \centering
    \includegraphics[scale=0.135]{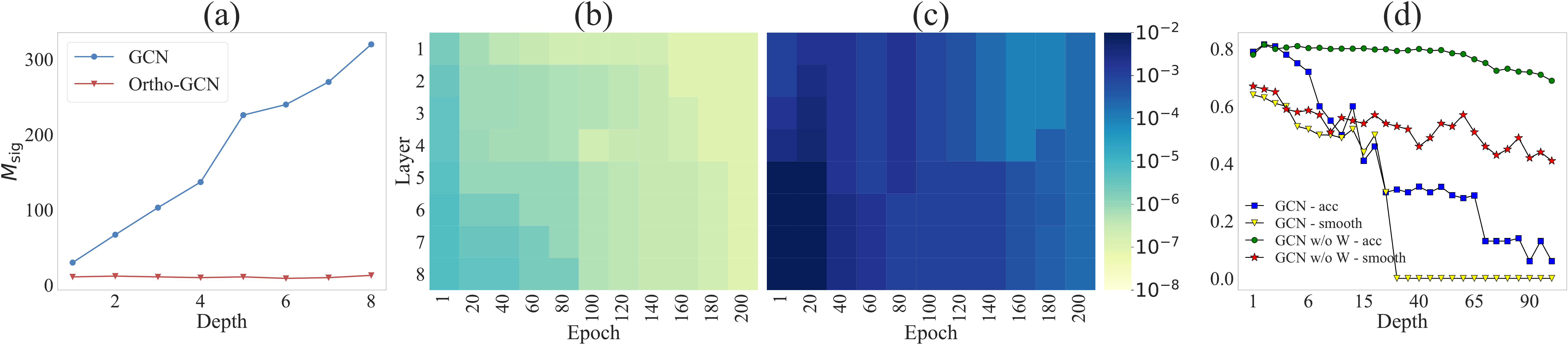}
    \caption{\textbf{(a)} The signal magnifications $M_{\mathrm{sig}}$ for GCNs with and without the orthogonal graph convolutions. \textbf{(b)} The gradient norms of vanilla GCN. \textbf{(c)} The  gradient norms of GCN augmented with orthogonal graph convolutions. \textbf{(d)} Test accuracy and node embedding smoothness under different model depths.}
    \label{fig: analysis}
\end{figure*}

\section{Forward and Backward Signaling Analysis}
In this section, we first introduce the notations and GNN model. We then analyze the forward and backward signaling process, and validate our argument with empirical studies. 

\subsection{Notations and GNNs} We denote matrices with boldface capital letters (e.g. $\bm{X}$), vectors with boldface lowercase letters~(e.g., $\bm{x}$) and scalars with lowercase alphabets~(e.g., $x$). An undirected graph is represented by $G=(\mathcal{V}, \mathcal{E})$, where $\mathcal{V} = \{v_i\}$ and $\mathcal{E} = \{(v_{i}, v_{j})\}$ denote the node and edge sets, respectively. Let $\bm{X} \in \mathbb{R}^{n \times d}$ denote the node feature matrix, where the $i$-th row is the corresponding $d$-dimensional feature vector of node $v_i$. The adjacency matrix is defined as $\boldsymbol{A} \in \mathbb{R}^{n \times n}$, which associates each edge $(v_{i}, v_{j})$ with its element $A_{i j}$; and $\bm{D}$ is the degree matrix. Let $ \tilde{\bm{A}}:= \bm{A} + \bm{I}_n$ and $\tilde{\bm{D}} := \bm{D} + \bm{I}_n$ be the adjacency and degree matrices of the graph augmented with self-loops. The normalized adjacency matrix is given by $\hat{\bm{A}} := \tilde{\bm{D}}^{-\frac{1}{2}}\tilde{\bm{A}}\tilde{\bm{D}}^{-\frac{1}{2}}$, which is widely used for spatial neighborhood aggregation in GNN models.

We use the graph convolutional networks (GCNs)~\cite{kipf2016semi} as a typical example, to illustrate how the traditional GNNs conduct node representation learning and to explain the model stability problem in the following subsection. The forward inference at the $l$-th layer of GCNs is formally defined as:

\begin{equation}
\bm{H}^{(l)}=\sigma(\hat{\bm{A}} \bm{H}^{(l-1)} \bm{W}^{(l)}),
\label{eq:GCN}
\end{equation}
where $\bm{H}^{(l)}$ denotes the node embedding matrix at the $l$-th layer; $\bm{H}^{(0)}$ is given by $\bm{X}$; 
$\sigma(\cdot)$ is the nonlinear activation function, such as $\mathrm{ReLU}$; $\boldsymbol{W}^{(l)} \in \mathbb{R}^{d \times d}$ is the linear transformation matrix. It is observed that graph convolutions consist of two key steps: the spatial neighbor aggregation based upon matrix $\hat{\bm{A}}$ and the feature transformation with matrix $\bm{W}^{(l)}$. Let $L$ denote the model depth. The output embedding of node $v_i$, i.e., $\bm{h}^{(L)}_i$ at the $i$-th row of $\bm{H}^{(L)}$, could be used to conduct the node or graph classification task. 

\subsection{Forward and Backward Steadiness Analysis}
For the spatial GNN models based on Eq.~\eqref{eq:GCN}, it is commonly known that stacking more layers tends to degrade the downstream task performance (e.g., node classification accuracy) significantly. Such phenomenon is usually attributed to the over-smoothing issue~\cite{li2018deeper}, which states that the node embeddings become similar due to the recursive neighbor aggregation. By simplifying the feature transformation and non-linear activation, many theoretical studies have been delivered to explain the recursive neighbor aggregation with low-passing filtering or Markov chains, which proves the node embeddings will converge to a unique equilibrium when the model depth keeps increasing~\cite{liu2020towards, nt2019revisiting}. Following this analysis, a variety of heuristic models are proposed to improve the neighbor aggregation and relieve the over-smoothing issue. For example, the skip connection combines with node embeddings from the previous layer to preserve the initial node features~\cite{chen2020simple}, and the edge dropping randomly removes edges to avoid the overwhelming amount of neighborhood information~\cite{rong2019dropedge}. 

We rethink the performance degrading problem of GNNs with an empirical observation -- vanilla GNNs even with a few layers suffer from the similar issue with the deep GNNs. Specifically, we investigate GCNs with the different model depths in terms of node classification accuracy and graph smoothness. The graph smoothness is defined as the average distance of node pairs~\cite{liu2020towards}, i.e,  $D = \frac{1}{|\mathcal{V}|^2}\sum_{v_i, v_j\in \mathcal{V}} ||\bm{h}^{(L)}_i - \bm{h}^{(L)}_j||$. As shown in Figure~\ref{fig: analysis}.(d), on dataset Cora~\cite{sen2008collective}, the accuracy and graph smoothness drop quickly when model depth $L$ is slightly enhanced to $8$. Even worse, metric $D$ approximates to zero once $L>20$, where GCNs fall into the random prediction. Such observations are in contrast to the previous theoretical analysis on the over-smoothing issue -- node embeddings become indistinguishable only if $L\rightarrow\infty$. By simply removing the feature transformation module in GCNs, it is surprising to observe that the accuracy and graph smoothness are maintained steadily until $L=100$. This makes us question the application of over-smoothing theory to explain the performance deterioration for shallow GNNs, since the node embeddings averaged by the standalone neighbor aggregation are well separated when $L$ is small. Thus, we shift the research attendance to the feature transformation has been ignored before, and argue that the unstable feature transformation is the primary factor to compromise GNNs. The stability of feature transformation is defined by both the forward inference steadiness and the backward gradient steadiness, which are introduced in the following.

\paragraph{Forward inference steadiness.} Keeping the steady forward inference is a ubiquitous technique to constrain the magnitudes of propagated signals and train deep neural networks stably~\cite{DBLP:conf/iclr/TrockmanK21}. 
Recalling the vanilla graph convolutions in Eq.~\eqref{eq:GCN}, 
the feature transformation may amplify the magnitudes of node embeddings without the proper constraint on matrix $\bm{W}^{(l)}$. Such magnitude amplification accumulates exponentially with layers~\cite{xie2017all}, and gives rise to the indefiniteness and randomness of forward node embeddings. The dramatical shift of internal signal distributions among the different layers could prevent the underlying model from training efficiently~\cite{ioffe2015batch}. To quantify the magnitude amplification of node embeddings across the whole GNN model, we define the signal magnification as:
\begin{equation}
    \label{equ:forward_ratio}
    M_{\mathrm{sig}} = \frac{1}{|\mathcal{V}|}\sum_{v_i\in \mathcal{V}} \frac{||\bm{h}^{(L)}_i||_2}{ ||\bm{h}^{(0)}_i||_2}.
\end{equation}
Specifically, metric $M_{\mathrm{sig}}$ averages the ratios of node embedding norms from the last layer over those at the initial layer. A larger value of $M_{\mathrm{sig}}$ indicates that the node embedding magnitudes are amplified excessively during the forward inference. 
Based on the common assumption that the initial data are whitened and decorrelated, the ideal $M_{\mathrm{sig}}$ should be $1$ to ensure the identical embedding magnitudes and signal distributions among layers.

\paragraph{Backward gradient steadiness.}
Besides the forward inference, an alternative direction to stabilize the training process is to maintain the backward gradient steadiness. 
GNNs are trained by gradient descent with back propagation to update weight $\bm{W}^{(l)}$ involved in the feature transformation.
While the existing studies focus on the forward message passing in GNNs, understanding backward gradient trajectory to optimize the feature transformation has remained limited. 
Therefore, we make the initial step to analyze the gradients in terms of parameter $\bm{W}^{(l)}$. To facilitate the gradient analysis, we simplify the non-linear activation function in Eq.~\eqref{eq:GCN} and obtain $\bm{H}^{(l)}=\hat{\bm{A}} \bm{H}^{(l-1)} \bm{W}^{(l)}$. It has been widely recognized that GNNs with and without non-linear have comparable node classification performance and learning curves~\cite{wu2019simplifying}. The model simplification helps us understand the gradient dynamics intuitively. 

\begin{theorem}\label{theorem_1}
Given the linear GNN model with $L$ layers and the specific training loss $\mathcal{L}$, the gradient with respect to parameter $\bm{W}^{(l)}$ at the $l$-th layer is given by:
\begin{equation}
\begin{aligned}
\label{eq: gradient}
 \frac{\partial \mathcal{L}}{\partial \bm{W}^{(l)}} 
 & =  (\bm{H}^{(l-1)})^\top(\hat{\bm{A}}^\top)^{L-l+1}\frac{\partial \mathcal{L}}{\partial \bm{H}^{(L)}} \\
 &\quad \cdot (\bm{W}^{(l+1)}\cdots\bm{W}^{(L)})^\top \in \mathbb{R}^{d\times d}.
\end{aligned}
\end{equation}
\end{theorem}

We give the detailed proof in Appendix.1. $\mathcal{L}$ could be represented by the cross-entropy loss in the node or graph classification task. 
According to Eq.~\eqref{eq: gradient}, it is observed that the backward gradients are backpropagated through the neighborhood aggregation and feature transformation, which are similar to the forward inference process. To update parameter $\bm{W}^{(l)}$ at layer $l$, the initial gradient $\frac{\partial \mathcal{L}}{\partial \bm{H}^{(L)}}$ is smoothed through the posterior
$L-l$ layers, and transformed with $(\hat{\bm{A}} \bm{H}^{(l-1)})^\top$. Since parameters $\{\bm{W}^{(l+1)},\cdots,\bm{W}^{(L)}\}$ are penalized during training, such smoothing and transformation will make most of the gradient entries approximate to zeros. In other words, the backward gradients may be vanishing at the initial few layers, which hampers the effective training of GNNs. To study the influence of gradient vanishing, we propose to apply gradient norms, i.e., $||\frac{\partial \mathcal{L}}{\partial \bm{W}^{(l)}}||_F$, $l=1, \cdots, L$, to quantify the gradient steadiness. The appropriate strengths of gradient norms would be preferred to stabilize the model training.

\paragraph{Illustration of model steadiness metrics.}
To validate our argument that the unstable feature transformation is the primary factor compromising the performances of shallow models, we empirically analyze the two proposed steadiness metrics, i.e, the forward signal magnification $M_{\mathrm{sig}}$ and the backward gradient norm $||\frac{\partial \mathcal{L}}{\partial \bm{W}^{(l)}}||_F$. Specifically, we develop a series of GCNs with different depths $L$, and show their signal magnifications in  Figure~\ref{fig: analysis}.(a). It is observed that metric $M_{\mathrm{sig}}$ rises quickly with depth $L$, which indicates the magnitudes of forward node embeddings are amplified exponentially across the model. The resulting butterfly effect of internal distribution shift prevents the effective training, which help explain the performance degrading of shallow GNNs. To illustrate the gradient dynamics from the final to the initial layer, we plot their gradient norms for an $8$-layer GNN in  Figure~\ref{fig: analysis}.(b). Being consistent with our theoretical analysis, $||\frac{\partial \mathcal{L}}{\partial \bm{W}^{(l)}}||_F$ shows decreasing tendency during the backpropagating process at the initial training phase. After hundreds of epochs, all gradients at the different layers were vanishing, which stops GNNs moving to the global minimum of the loss landscape. It should be noted that the unstable forward and backward signaling appear in shallow GCN with $8$ layers, while GCNs without the feature transformation modules deliver steady performances until $L=100$ as illustrated in  Figure~\ref{fig: analysis}.(d). Therefore, we could conclude that the unstable feature transformation is responsible for the performance deterioration of shallow GNNs.

\section{Orthogonal Graph Convolutions}
To overcome the issues of unstable forward inference and backward gradients, we explore the use of orthogonality on the feature transformation. Although many other methods have been applied to constrain the forward direction in GNNs, such as pair or group normalization~\cite{zhao2019pairnorm, zhou2020towards}, they fail to guarantee the backward steadiness due to the complexity variations of signals~\cite{saxe2013exact}. In this section, we first review the orthogonal matrix and its theoretical properties in stabilizing forward and backward signaling. We then discuss the challenges of applying orthogonal constraint in GNNs, and optimize the orthogonal graph convolutions to tailor to the graph data.

\subsection{Orthogonality}

A matrix $ \bm{W} \in \mathbb{R}^{d \times d}$ is orthogonal if $ \bm{W}^\top \bm{W} = \bm{W} \bm{W}^\top = \bm{I}$. Encouraging the orthogonality in deep neural networks has proven to yield several benefits, such as stable and quick training~\cite{xiao2018dynamical}, better generalization~\cite{bansal2018can}, and improved robustness against the adversarial attack~\cite{tsuzuku2018lipschitz}. In GNNs, we concern the preferred properties of orthogonal feature transformation to stabilize the signaling processes at both the forward and backward directions. 
We thus ignore the neighbor aggregation and non-linear activation at each layer to simplify the theoretical analysis, and consider their impacts empirically in the following model design. Mathematically, for the $l$-th layer, the graph convolutions could be simplified as $\bm{H}^{(l)} = \hat{\bm{H}}^{(l)}\bm{W}^{(l)}$, where $\hat{\bm{H}}^{(l)} = \hat{\bm{A}} \bm{H}^{(l-1)}$ denotes the node embeddings after neighbor aggregation. The following theorem, proved in~\cite{huang2018orthogonal}, shows that orthogonal weight $\bm{W}^{(l)}$ could preserve the norms of forward embeddings and backward gradients for signals passing through the feature transformation module.

\begin{theorem}\label{theorem_2}
Let $\bm{W}^{(l)} \in \mathbb{R}^{d \times d}$ denote the orthogonal matrix adopted by the feature transformation at the $l$th layer. Let $\hat{\bm{h}}^{(l)}$ and $\bm{h}^{(l)}$ denote the node embeddings, which are given by specific rows in matrices $\hat{\bm{H}}^{(l)}$ and $\bm{H}^{(l)}$, respectively. (1) Assume the mean of $\hat{\bm{h}}^{(l)}$ is $\mathbb{E}_{\hat{\bm{h}}^{(l)}}[\hat{\bm{h}}^{(l)}]=\mathbf{0}$, and the
covariance matrix of $\hat{\bm{h}}^{(l)}$ is $\operatorname{cov}(\hat{\bm{h}}^{(l)})=\sigma^{2} \bm{I}$. Then $\mathbb{E}_{\bm{h}^{(l)}}[\bm{h}^{(l)}]=\mathbf{0}$, and $\operatorname{cov}(\bm{h}^{(l)})=\sigma^{2} \bm{I}$. (2) We have $\|\bm{H}^{(l)}\|_F=\|\hat{\bm{H}}^{(l)}\|_F$.
(3) Given the back-propagated gradient $\frac{\partial \mathcal{L}}{\partial \bm{H}^{(l)}}$, we have $\left\|\frac{\partial \mathcal{L}}{\partial \bm{H}^{(l)}}\right\|_F=\left\|\frac{\partial \mathcal{L}}{\partial \hat{\bm{H}}^{(l)}}\right\|_F$.
\end{theorem}
We list the detailed proof in Appendix.2. Theorem~\ref{theorem_2} shows the benefits of orthogonal feature transformation to stabilize simplified GNNs: (1) The Frobenius norms of node embeddings $\bm{H}^{(l)}$ and $\hat{\bm{H}}^{(l)}$ are kept to be identical, which helps constrain the embedding magnitudes across model and approximate the desired signal magnification $M_{\mathrm{sig}}$ with value $1$. (2) The norms of backward gradients are maintained when passing through the feature transformation layer. This relieves the gradient vanishing issue as studied in Theorem~\ref{theorem_1}.

\subsection{Orthogonal Graph Convolutions}
To ensure orthogonality, 
most of the prior efforts either insert an extra orthogonal layer to transform matrix $\bm{W}^{(l)}$~\cite{DBLP:conf/iclr/TrockmanK21}, or exploit orthogonal weight initialization to provide model a good start~\cite{xiao2018dynamical}. However, it is non-trivial to directly apply the existing orthogonal methods due to two challenges in GNNs. First, since the node features usually contain the key information for the downstream task, the intuitive orthogonal initialization will accelerate the training process toward local minimum and damage the model's learning ability. Second, even with the strict orthogonality of $\bm{W}^{(l)}$, the impacts brought by the neighbor aggregation and non-linear activation make it failed to preserve the embedding norms of the successive layers. According to Theorem~\ref{theorem_2}, the orthogonality only obtains $\|\bm{H}^{(l)}\|_F=\|\hat{\bm{H}}^{(l)}\|_F$ within the same layer in simplified GNNs, instead of strictly ensuring $\|\bm{H}^{(l)}\|_F=\|\bm{H}^{(l-1)}\|_F$ at the successive layers in non-linear GNNs. To bridge the gaps, we propose the orthogonal graph convolutions, named Ortho-GConv, by optimizing the orthogonality designs comprehensively from three architectural perspectives, including hybrid weight initialization, orthogonal transformation, and orthogonal regularization. The details are introduced in the following. 

\paragraph{Hybrid weight initialization.}

It is widely demonstrated that GNNs tend to overfit on the large and attributed graph data~\cite{kong2020flag}. Although the orthogonal initialization allows training the deep vanilla neural networks efficiently, the quick convergence may iterate to the local optimum and intensifies the overfitting problem. To obtain the trade-off between orthogonality and model's learning ability, we propose the hybrid weight initialization to set weight at the $l$-th layer as follows:

\begin{equation}
\bm{Q}^{(l)} = \beta \bm{P}^{(l)} + (1-\beta) \bm{I} \in \mathbb{R}^{d\times d}.
\end{equation}
$\bm{P}^{(l)}$ is initialized by the traditional random approaches (e.g., Glorot initialization~\cite{glorot2010understanding}), while we adopt the identity initialization~\cite{le2015simple}, the simplest orthogonal method, to obtain the initial orthogonality $\bm{I}$. $\beta$ is a hyperparameter to control the trade-off.

\paragraph{Orthogonal transformation.}
Given initialized weight $\bm{Q}^{(l)}$, we adopt an extra orthogonal transformation layer to transform it and improve the orthogonality before its appliance for the feature transformation. We use Newton's iteration to illustrate our method due to its simplicity~\cite{DBLP:conf/cvpr/Huang00WYL020}. Specifically, the orthogonal transformation based on Newton's iteration is divided into two steps: spectral bounding and orthogonal projection.
First, the spectral bounding normalizes weight $\bm{Q}^{(l)}$ as: $\hat{\bm{Q}}^{(l)} =\frac{\bm{Q}^{(l)}}{\|\bm{Q}^{(l)}\|_{F}}$. 

The orthogonal projection then maps matrix $\hat{\bm{Q}}^{(l)}$ to obtain the orthogonal weight $\bm{W}^{(l)}$.  Mathematically, the orthogonal projection is given by: $\bm{W}^{(l)}=\bm{M}^{-\frac{1}{2}}\hat{\bm{Q}}^{(l)}$, where $\bm{M}=\hat{\bm{Q}}^{(l)} \hat{\bm{Q}}^{(l)T}$ is the covariance matrix. Due to the exponential complexity to compute the square root of covariance matrix, $\bm{M}^{-\frac{1}{2}}$ is instead computed by the Newton's iteration with the following iterative formulas:
\begin{equation}
\left\{\begin{array}{l}
\bm{B}_{0}=\bm{I} \\
\bm{B}_{t}=\frac{1}{2}\left(3 \bm{B}_{t-1}-\bm{B}_{t-1}^{3} \bm{M}\right), \quad t=1,2, \ldots, T
\end{array}\right.
\end{equation}
where $T$ is the number of iterations. Under the condition of $\|\bm{I}-\bm{M}\|_{2}<1$, it has been proved that $\bm{B}_{T}$ will converge to $\bm{M}^{-\frac{1}{2}}$. Therefore, we instead obtain the orthogonal weight via $\bm{W}^{(l)}=\bm{B}_{T} \hat{\bm{Q}}^{(l)}$. Weight $\bm{W}^{(l)}$ is applied for the feature transformation as shown in Eq.~\eqref{eq:GCN}.

\paragraph{Orthogonal regularization.} 
Even accompanied with orthogonal matrix  $\bm{W}^{(l)}$ in the feature transformation, the norms of forward node embeddings still fails to be preserved due to the neighbor aggregation and non-linear activation in GNNs. To be specific, recalling the graph convolutitons in Eq.~\eqref{eq:GCN}, we have:

\begin{equation}
\begin{aligned}
    \|\bm{H}^{(l)}\|_F & =\|\sigma(\hat{\bm{A}} \bm{H}^{(l-1)} \bm{W}^{(l)})\|_F \\
    & \leq \|\hat{\bm{A}} \bm{H}^{(l-1)} \bm{W}^{(l)}\|_F = \|\hat{\bm{A}} \bm{H}^{(l-1)}\|_F\\
    & \leq \|\hat{\bm{A}}\|_F
\|\bm{H}^{(l-1)}\|_F \leq \|\bm{H}^{(l-1)}\|_F.
\end{aligned}
\end{equation}
The first inequality holds since the non-linear activation of ReLU maps the negative entries into zeros. The following equality is obtained by the norm-preserving property of orthogonal weight $\bm{W}^{(l)}$. Since the entries in adjacency matrix $\hat{\bm{A}}$ are normalized within range $[0,1]$, we have $\|\hat{\bm{A}}\|_F \leq 1$ to get the final inequality.
 
Compared to the forward embedding explosion in vanilla GNNs, such norm vanishing will also shift the internal embedding distributions and result in the inefficient training. To maintain the node embedding norms during the forward inference, we propose a simple orthogonal regularization to constrain weight $\bm{W}^{(l)}$ as:

\begin{equation}
\mathcal{L}_{\text {auxiliary }}=\lambda \sum_{l}\left\|\bm{W}^{(l)}{\bm{W}^{(l)}}^{\top}-c^{(l)}\cdot\bm{I}\right\|_F,
\end{equation}
where $\lambda$ is a hyperparameter. $c^{(l)}$ is the trainable scalar to control the norm of weight $\bm{W}^{(l)}$. We initialize $c^{(l)}$ with value $1$, and let model automatically learn how to preserve the embedding norms in the forward inference. A larger $c^{(l)}$ indicates to compensate the norm vanishing brought by the neighbor aggregation and non-linear activation. 

Our Ortho-GConv is a general module capable of  augmenting the existing GNNs. Without loss of generality, we adopt the simple identity weight and Newton's iteration to provide the orthogonal initialization and transformation, respectively. More the other orthogonal methods could be applied to further improve model's performance in the future.

\subsection{Steadiness Study of Our Model}
To validate the effectiveness of our Ortho-GConv in stabilizing the forward and backward signaling, we implement it upon the vanilla GCNs. While the signal magnifications $M_{\mathrm{sig}}$ are constrained around value $1$ for models of different depths in Figure~\ref{fig: analysis}.(a), the gradient norms $||\frac{\partial \mathcal{L}}{\partial \bm{W}^{(l)}}||_F$ are comparable for the different layers within an $8$-layer model in Figure~\ref{fig: analysis}.(c). In other word, our Ortho-GConv could constrain the magnitudes of node embedding to stabilize the forward inference, meanwhile maintaining the gradients at the backward direction to update model effectively.

\section{Experiments}
In this section, we conduct extensive experiments to evaluate our model aimed at answering the following questions:
\begin{itemize}
\item \textbf{Q1}: 
How effective is the proposed Ortho-GConv applied to current popular GNNs on the full-supervised node classification and graph classification tasks?

\item \textbf{Q2}: How much do the hybrid weight initialization, orthogonal transformation and orthogonal regularization influence the Ortho-GConv?

\end{itemize}

\noindent We provide more details of the results and analysis on semi-supervised node classification in Appendix.5.

\paragraph{Benchmark Datasets.} 

For full-supervised node classification task, we use Cora~\cite{sen2008collective}, CiteSeer~\cite{sen2008collective}, PubMed~\cite{sen2008collective}, and three sub-sets of WebKB~\cite{pei2020geom}: Cornell, Texas and Wisconsin to evaluate the performance. 
For graph classification task, we use the protein datasets including D\&D~\cite{dobson2003distinguishing} and PROTEINS~\cite{DBLP:conf/ismb/BorgwardtOSVSK05}. 
In addition, we conduct experiments on ogbn-arxiv~\cite{hu2020open} to evaluate the scalability and performance of Ortho-GConv on large graphs. 
The statistics of datasets and parameter settings are provided in Appendix.3 and 4, respectively.

\subsection{Full-supervised Node Classification}

We construct two experiments to evaluate the performance of our proposed Ortho-GConv, namely the comparison with different layers and the comparison with SOTA. We apply the same full-supervised experimental setting and training setting with GCNII~\cite{chen2020simple}, and use six datasets, including Cora, Citeseer, Pubmed, Cornell, Texas and Wisconsin to evaluate the performance. For each dataset, we randomly split the nodes of each class into 60\%, 20\% and 20\% for training, validation and testing, respectively. In addition, we conduct experiments on ogbn-arxiv to further evaluate the performance of our proposed Ortho-GConv on large graph. For this dataset, we split the nodes of each class into 54\%, 18\% and 28\% for training, validation and testing by following the previous effort~\cite{DBLP:conf/nips/HuFZDRLCL20}.

\textbf{(1) Comparison with Different Layers.} We take GCN, JKNet and GCNII as the three backbones and compare our proposed Ortho-GConv with their original models at 2/4/8 layers respectively. The experimental results are shown in Table \ref{tab:full_overall}. 
We make the following conclusions: 

\romannumeral1) Ortho-GConv generally improves the performances of all the backbones on each dataset under different numbers of layers. For example, Ortho-GConv delivers average improvements of 1.9\% and 2.1\% over the backbones with 2 layers on Cora and Citeseer, respectively; while achieving remarkable gains of 2.8\% and 12.4\% in the case of 8 layers.

\romannumeral2) With the increase of the number of layers, the performances of GNN models decrease significantly. However, our proposed Ortho-GConv achieves the comparable performances for 2 and 8 layers, which relieve the performance degrading of shallow GNNs as analyzed before. This phenomenon is attributed to the advantage of our Ortho-GConv that can solve the problem of gradient vanishing and make the model inference stable.

\romannumeral3)Ortho-GConv also generally achieves better performances on ogbn-arxiv than the backbones. The results demonstrate that Ortho-GConv is applicable to large graph.

Therefore, our proposed Ortho-GConv could tackle the problems of forward inference explosion and backward gradient vanishing problems in the vanilla shallow GNNs, which enable them to be trained steadily.

\textbf{(2) Comparison with SOTA.} 

In order to verify the overall performance of Ortho-GConv, we select the best performance from each backbone with Ortho-GConv and compare with the current popular SOTA methods.

We follow the same setting with GCNII and repeat the experiment 5 times, and report the average results in Table~\ref{tab:full_sota}. 
The experimental results reveal that our proposed method outperforms all the baselines. 
Ortho-GCNII obtains 2.2\% average improvement on all datasets. Especially, Ortho-GConv delivers an improvement of 7.4\% over GCNII on Texas.
In addition, the performance of few-layer GCNII model equipped with Ortho-GConv is better than deep GNN models, which demonstrates the superiority of the Ortho-GConv. More details about the results with layer number information are reported in Appendix.6.

\begin{table}[htp]
\setlength{\abovecaptionskip}{0.cm}
\setlength{\tabcolsep}{4pt}
  \centering
  \small
  \caption{Accuracy (\%) comparisons with SOTA on full-supervised tasks. The highest performances are in bold. }
  \
  \label{tab:full_sota}
    \begin{tabular}{ccccccc}
    \toprule
    Method&  Cora  &Cite.  &Pumb.   &Corn.  &Texa.  &Wisc.\cr
    \midrule
    GCN           &85.77  &73.68  &87.91  &57.84  &55.68  &49.02  \cr
    GAT         &86.37  &74.32  &87.62  &54.32  &58.38  &49.41 \cr
    Geom-GCN           &85.19  &{\bf77.99}  &90.05   &56.76  &57.58  &58.24  \cr
    APPNP         &87.87  &76.53  &89.40  &73.51  &65.41  &69.02 \cr
    JKNet           &86.20  &75.89  &89.55  &62.16  &62.70 &64.31  \cr
    Incep(DropEgde)           &86.86  &76.83   &89.18   &61.62  &57.84    &50.20 \cr
    GCNII           &88.49   &77.08    &89.78    &74.86   &71.46    &75.30 \cr
    Ortho-GCNII           &{\bf 88.81}   &77.26    &{\bf90.30}    &{\bf76.22}   &{\bf77.84}   &{\bf77.25}  \cr
    \bottomrule
    \end{tabular}
\end{table}

\begin{table*}[htp]
\setlength{\abovecaptionskip}{0.cm}
    \small
    \centering
  \caption{Mean classification accuracy (\%) of full-supervised node classification comparisons on different backbones with/without Ortho-GConv on full-supervised node classification tasks}
  \label{tab:full_overall}
\centering
\begin{tabular}{ll|rrrrrr}
\toprule
\multirow{2}{*}{Dataset}  & \multirow{2}{*}{Backbone} & \multicolumn{2}{c}{2 layers} & \multicolumn{2}{c}{4 layers} & \multicolumn{2}{c}{8 layers} \\
                          &                           & \multicolumn{1}{c}{original~} & \multicolumn{1}{c}{Ortho-GConv} & \multicolumn{1}{c}{original~} & \multicolumn{1}{c}{Ortho-GConv} & \multicolumn{1}{c}{original~} & \multicolumn{1}{c}{Ortho-GConv}  \\
\hline
\multirow{3}{*}{Cora}     & GCN                       &85.77$\pm$1.73                         & {\bf87.28$\pm$1.68}                           &82.37$\pm$2.47                           & {\bf86.20$\pm$1.75}                              &81.13$\pm$2.78                            & {\bf85.27$\pm$1.64}                               \\
                         & JKNet                     &85.96$\pm$1.54                           &{\bf87.36$\pm$1.74}                           &86.20$\pm$1.45                            &{\bf87.12$\pm$2.25}                             &85.84$\pm$1.64                          &{\bf87.24$\pm$2.09}                            \\
                          & GCNII                     &86.28$\pm$0.79                           & {\bf88.49$\pm$1.59}                              &85.70$\pm$2.10             & {\bf88.81$\pm$1.69}                             &86.80$\pm$2.10                          & {\bf88.41$\pm$1.43}                            \\
\hline 
\multirow{3}{*}{Citeseer} & GCN                       &73.68$\pm$1.69                          & {\bf75.59$\pm$1.78}                            & 68.03$\pm$5.90                            & {\bf74.80$\pm$1.11}                             & 53.10$\pm$6.34                            & {\bf71.61$\pm$2.46}                              \\
                          & JKNet                     &75.89$\pm$1.54                         & {\bf76.89$\pm$1.64}                            &74.97$\pm$1.76                           &{\bf76.11$\pm$1.76}                             &74.85$\pm$1.69                          &{\bf75.60$\pm$1.95}                           \\
                          & GCNII                     &75.31$\pm$2.36                            & {\bf77.26$\pm$1.83}                              &75.60$\pm$1.70                            &{\bf76.94$\pm$2.10}                              &76.10$\pm$2.10                          & {\bf77.11$\pm$2.20}                             \\
\hline
\multirow{3}{*}{Pubmed}   & GCN                       &{\bf87.91$\pm$0.44}                       & 86.04$\pm$0.61               & 77.00$\pm$7.55                            & {\bf84.68$\pm$0.55}                              & 69.49$\pm$0.98                            & {\bf83.75$\pm$0.50}                              \\
                          & JKNet                     & 89.40$\pm$0.30                            & {\bf89.46$\pm$0.28}                           &89.47$\pm$0.44                           &{\bf 89.54$\pm$0.38}                             &89.55$\pm$0.47                          &{\bf 89.57$\pm$0.21}   \\
                          & GCNII                     &89.51$\pm$0.69                            &{\bf90.30$\pm$0.30}                             &89.50$\pm$0.40                            &{\bf90.04$\pm$0.32}                             &89.78$\pm$0.33           &{\bf89.80$\pm$0.43}                             \\
\hline
\multirow{3}{*}{Corn.}   & GCN                       & {\bf52.70$\pm$5.05}                         & 58.38$\pm$3.62                           &57.84$\pm$3.00                            & {\bf57.84$\pm$3.08}                          &57.84$\pm$3.00                            & {\bf57.84$\pm$3.08}                           \\
                          & JKNet                     & 62.16$\pm$5.05                            &{\bf63.24$\pm$4.90}                            &52.97$\pm$11.6                            & {\bf58.92$\pm$4.44}                              &56.22$\pm$7.97                          &{\bf58.92$\pm$5.53}  \\
                          & GCNII                     &58.92$\pm$4.44                           & {\bf74.05$\pm$4.09}                              &66.00$\pm$6.20                            & {\bf75.14$\pm$6.72}                              &74.10$\pm$5.60                          & {\bf75.14$\pm$7.13}                           \\
\hline
\multirow{3}{*}{Texa.}   & GCN                       &55.14$\pm$7.78                       & {\bf61.08$\pm$9.08}                          & 55.68$\pm$5.60                            & {\bf58.92$\pm$7.49}                              & 54.59$\pm$7.00                            & {\bf58.38$\pm$6.22}                            \\
                          & JKNet                     &58.38$\pm$6.22                           &{\bf61.08$\pm$8.01}                             &62.70$\pm$4.00                            &{\bf62.70$\pm$5.85}                              &{\bf62.16$\pm$5.05}            &60.54$\pm$3.52   \\
                          & GCNII                     &69.73$\pm$13.30                        &{\bf77.84$\pm$6.72}                          &71.40$\pm$8.00                            &{\bf77.30$\pm$5.60}                             &70.80$\pm$5.20                &{\bf75.14$\pm$3.52}                       \\
\hline
\multirow{3}{*}{Wisc.}   & GCN                       & 49.02$\pm$3.66                          &{\bf50.59$\pm$8.40}                          & 46.67$\pm$7.76                            & {\bf48.24$\pm$9.46}                              & 40.00$\pm$10.6                            & {\bf46.67$\pm$9.44}                              \\
                          & JKNet           & 64.31$\pm$6.41                      & {\bf69.41$\pm$5.64}                             & 59.61$\pm$4.06                            &{\bf68.63$\pm$4.80}                             &56.86$\pm$3.10                          & {\bf65.88$\pm$5.11}   \\
                          & GCNII      &72.94$\pm$9.23     & {\bf77.25$\pm$3.54}                 &74.50$\pm$7.80     &{\bf77.25$\pm$5.11}          &{\bf75.30$\pm$8.10}        & 75.29$\pm$6.29                            \\
\hline
\multirow{2}{*}{ogbn-arxiv}   & GCN                       &71.28$\pm$0.28                          &  {\bf71.33$\pm$0.26}                           & 72.30$\pm$0.17                           & {\bf72.30$\pm$0.13}                              &71.84$\pm$0.27                  & {\bf71.87$\pm$0.12}                              \\
                                     
                                       & GCNII                    &71.24$\pm$0.19                   &{\bf71.35$\pm$0.21}                              &  71.21$\pm$0.19                         & {\bf71.38$\pm$0.14}                               &{\bf72.51$\pm$0.28}                       &72.44$\pm$0.46                              \\
\bottomrule
\end{tabular}
\end{table*}

\subsection{Graph Classification}
For the graph classification task, we use Graph-U-Nets~\cite{gao2019graph} as the backbone, PSCN~\cite{niepert2016learning}, DGCNN\cite{zhang2018end}, and DiffPool\cite{DBLP:conf/nips/YingY0RHL18} as the comparison baselines, and conduct experiments on the D\&D and PROTEINS datasets to evaluate our model.
We follow the same experimental settings as Graph-U-Nets for fair comparison, and fix the parameters $T$ = 4, $\beta$ = 0.4.
The experimental results are reported in Table~\ref{tab:graph classification}. We can see that applying Ortho-GConv on the Graph-U-Nets (g-U-Nets) achieves new state-of-the-art performances on datasets D\&D and PROTEINS, which demonstrates the effectiveness of Ortho-GConv on the graph classification task.
\begin{table}[htp]
\setlength{\abovecaptionskip}{0.cm}
\setlength{\tabcolsep}{1.7pt}
  \centering
  \small
  \caption{Accuracy (\%) comparisons with SOTA on graph classification tasks. The highest performances are in bold. }
  \
  \label{tab:graph classification}
    \begin{tabular}{cccccc}
    \toprule
    Dataset&  PSCN  &DGCNN  &DiffPool   &g-U-Nets  &Ortho-g-U-Nets.\cr
    \midrule
        D\&D           &76.27  &79.37  &80.64  &83.00  &{\bf 83.87}   \cr
        PROTEINS           &75.00  &76.26  &76.25  &77.68  &{\bf 78.78}   \cr
    \bottomrule
    \end{tabular}
\end{table}

In summary, our proposed Ortho-GConv gains a signiﬁcant performance improvement over representative baseline methods on the node classification, and graph classification tasks, which answers the first question asked at the beginning of this section. 

\subsection{Ablation Study}

To study the importance of Ortho-GConv's three perspectives, we construct several variants of our model on the hybrid weight initialization, orthogonal transformation and orthogonal regularization, and define them as follows: 1) Ortho-GCN w/o initialization, omitting the hybrid weight initialization module from Ortho-GCN; 2) Ortho-GCN w/o transformation, omitting the orthogonal transformation from Ortho-GCN; 3) Ortho-GCN w/o regularization, omitting the orthogonal regularization from Ortho-GCN. To shed light on the contributions of the three perspectives, we report the semi-supervised node classification results of Ortho-GCN and their variants on Cora, as shown in Figure \ref{fig: ablation}.(a). We have the following observations:

\romannumeral1) Compared with GCN, Ortho-GCN achieves a better performance than the three ablated variants and the GCN model. It further demonstrates the importance of three orthogonal techniques for stabilizing the model.

\romannumeral2) Removing the orthogonal initialization from the proposed model has a considerable impact on the performance, which suggests that this component plays an important role in node classification task.

The experimental results further demonstrate the importance of the three key components of our model, which correspondingly answers the second question.

Besides, we provide an ablation study to show how the iteration number $T$ affects on the performance and training time of Ortho-GConv. We conduct experiments on Cora dataset using 2-layer and 8-layer Ortho-GCN model. The results are reported in Figure~\ref{fig: ablation}.(b) and (c), respectively.
We find that as $T$ increases, the time consumption becomes larger and larger. From the Figure~\ref{fig: ablation}.(b) and (c), we can also observe that larger iteration numbers and smaller iteration numbers reduce the performance of our proposed model. When the iteration number is 4, we get the best performance. In conclusion, appropriate $T$ is optimized to achieve high accuracy with acceptable time complexity.

\begin{figure}[htp]
    \centering
    \includegraphics[width=8.5cm]{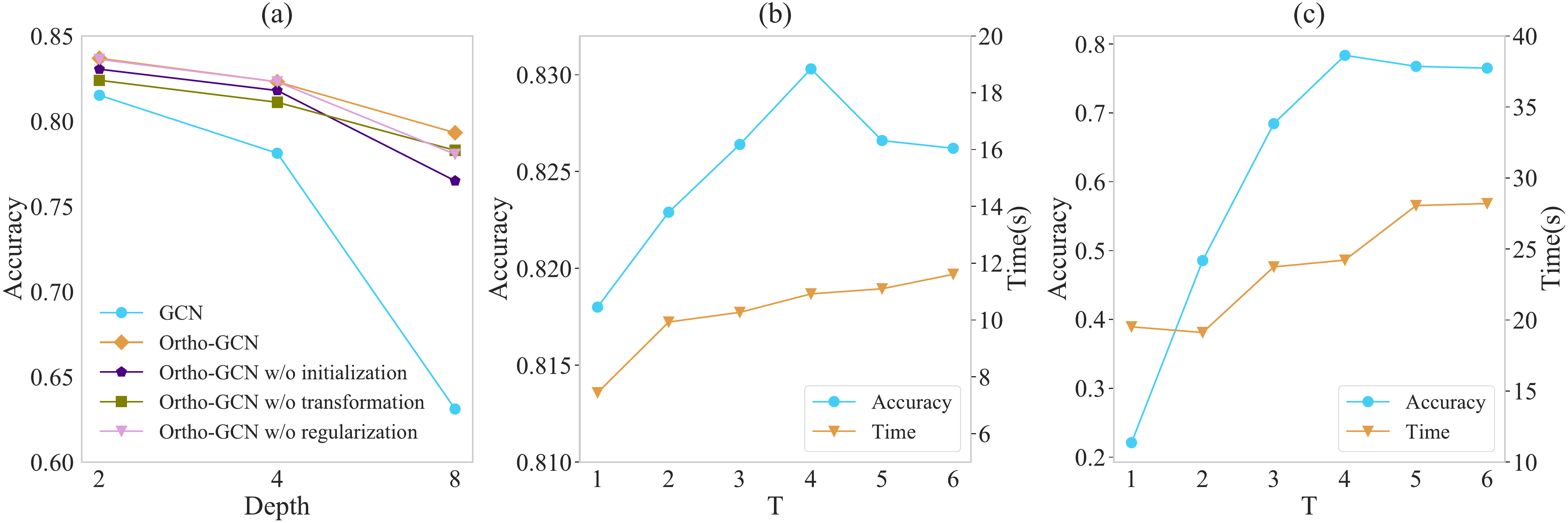}
    \caption{\textbf{(a)} Depth vs Accuracy for GCN, our model and three ablated models on Cora. \textbf{(b)} Effects of the iteration number $T$ of the 2-layer Ortho-GCN on Cora dataset. \textbf{(c)} Effects of the iteration number $T$ of the 8-layer Ortho-GCN on Cora dataset.}
    \label{fig: ablation}
\end{figure}

\section{Conclusion}

In this paper, we first conduct a series of analytical experiments to probe the reason for degrading performances of GNNs when stacked with more convolutional layers. We argue that the primary factor is the unstable forward and backward signaling in GNNs. Then, we propose a  orthogonal  graph  convolutions to augment the GNN backbones to stabilize the model training and improve the generalization performance of the model.

The experiments show that our Ortho-GConv achieves the better performance than SOTA methods on node and graph classification tasks. 

\bibliography{latex_2022}
\clearpage

\section{Appendix}

\subsection{Appendix 1: Proof for Theorem 1}

\begin{theorem}\label{theorem_1}
Given the linear GNN model with $L$ layers and the specific training loss $\mathcal{L}$, the gradient with respect to parameter ${{\bm{W}} ^{(l)}}$ at the $l$-th layer is given by:
\begin{equation}
 \begin{aligned}
 \label{eq: gradient}
 \frac{\partial \mathcal{L}}{\partial \bm{W}^{(l)}} 
 & =  (\bm{H}^{(l-1)})^\top(\hat{\bm{A}}^\top)^{L-l+1}\frac{\partial \mathcal{L}}{\partial \bm{H}^{(L)}} \\
 &\quad \cdot (\bm{W}^{(l+1)}\cdots\bm{W}^{(L)})^\top \in \mathbb{R}^{d\times d}.
 \end{aligned}
\end{equation}
\end{theorem}

\begin{proof} The forward inference at the $l$-th layer of GCNs without non-linear function is formally defined as:
\begin{equation}
\bm{H}^{(l)}=\hat{\bm{A}} \bm{H}^{(l-1)} \bm{W}^{(l)}.
\label{eq:GCN}
\end{equation}

We can see that the training loss at layer $l$, the derivative $\mathcal{L}$ with respect to $\bm{W}^{(l)}$ is:
\begin{equation}
 \begin{aligned}
 \frac{\partial \mathcal{L}}{\partial \bm{W}^{(l)}}& =
 \frac{\partial \bm{H}^{(l)}}{\partial \bm{W}^{(l)}} \frac{\partial \mathcal{L}}{\partial \bm{H}^{(l)}} \\
 & = (\hat{\bm{A}} \boldsymbol{H}^{(l-1)})^\top \frac{\partial \mathcal{L}}{\partial \bm{H}^{(l)}} \in \mathbb{R}^{d\times d}.
 \end{aligned}
 \label{eq: gradient2}
\end{equation}

Following the chain rule of gradient backpropagation in terms of matrix, we can calculate the$\frac{\partial \mathcal{L}}{\partial \bm{H}^{(l)}}\in\mathbb{R}^{n\times d}$ by following equation:
\begin{equation}
 \begin{aligned}
 \frac{\partial \mathcal{L}}{\partial \bm{H}^{(l)}} & = \hat{\bm{A}}^\top \frac{\partial \mathcal{L}}{\partial \bm{H}^{(l+1)}} {\bm{W}^{(l+1)}}^\top \\
 & = (\hat{\bm{A}}^\top)^{L-l}\frac{\partial \mathcal{L}}{\partial \bm{H}^{(L)}}{\bm{W}^{(L)}}^\top\cdots{\bm{W}^{(l+1)}}^\top
 \end{aligned}
\end{equation}
Following the above, the gradient in Eq.~\eqref{eq: gradient2} can be represented by:
\begin{equation}
 \begin{aligned}
 \label{eq: gradient}
  \frac{\partial \mathcal{L}}{\partial \bm{W}^{(l)}} 
  & =  (\bm{H}^{(l-1)})^\top(\hat{\bm{A}}^\top)^{L-l+1}\frac{\partial \mathcal{L}}{\partial \bm{H}^{(L)}} \\
  &\quad \cdot (\bm{W}^{(l+1)}\cdots\bm{W}^{(L)})^\top \in \mathbb{R}^{d\times d}.
 \end{aligned}
\end{equation}
\end{proof}

\subsection{Appendix 2: Proof for Theorem 2}

\begin{theorem}\label{theorem_2}
Let $\bm{W}^{(l)} \in \mathbb{R}^{d \times d}$ denote the orthogonal matrix adopted by the feature transformation at the $l$th layer. Let $\hat{\bm{h}}^{(l)}$ and $\bm{h}^{(l)}$ denote the node embeddings, which are given by specific rows in matrices $\hat{\bm{H}}^{(l)}$ and $\bm{H}^{(l)}$, respectively. 
(1) Assume the mean of $\hat{\bm{h}}^{(l)}$ is $\mathbb{E}_{\hat{\bm{h}}^{(l)}}[\hat{\bm{h}}^{(l)}]=\mathbf{0}$, and the covariance matrix of $\hat{\bm{h}}^{(l)}$ is $\operatorname{cov}(\hat{\bm{h}}^{(l)})=\sigma^{2} \bm{I}$. Then $\mathbb{E}_{\bm{h}^{(l)}}[\bm{h}^{(l)}]=\mathbf{0}$, and $\operatorname{cov}(\bm{h}^{(l)})=\sigma^{2} \bm{I}$. 
(2) We have $\|\bm{H}^{(l)}\|_F=\|\hat{\bm{H}}^{(l)}\|_F$.
(3) Given the back-propagated gradient $\frac{\partial \mathcal{L}}{\partial \bm{H}^{(l)}}$, we have $\left\|\frac{\partial \mathcal{L}}{\partial \bm{H}^{(l)}}\right\|_F=\left\|\frac{\partial \mathcal{L}}{\partial \hat{\bm{H}}^{(l)}}\right\|_F$.
\end{theorem}

\begin{proof}
(1) It's easy to calculate:
\begin{equation}
\begin{aligned}
    \mathbb{E}_{\bm{h}^{(l)}}[\bm{h}^{(l)}] & =\mathbb{E}_{\hat{\bm{h}}^{(l)}}[\bm{W}^{(l)}\hat{\bm{h}}^{(l)}] \\
    & =\bm{W}^{(l)}\mathbb{E}_{\hat{\bm{h}}^{(l)}}[\hat{\bm{h}}^{(l)}] \\
    & = \bm{W}^{(l)} \cdot 0\\ 
    & =\mathbf{0}.
\end{aligned}
\end{equation}

The covariance of $\bm{h}^{(l)}$ is calculated as follows:

\begin{equation}
    \begin{aligned}
    \operatorname{cov}(\bm{h}^{(l)}) &=\mathbb{E}_{\bm{h}^{(l)}}\left[\bm{h}^{(l)}-\mathbb{E}_{\bm{h}^{(l)}}[\bm{h}^{(l)}]\right]^{2} \\
    &=\mathbb{E}_{\hat{\bm{h}}^{(l)}}\left[\bm{W}^{(l)}\left(\hat{\bm{h}}^{(l)}-\mathbb{E}_{\hat{\bm{h}}^{(l)}}[\hat{\bm{h}}^{(l)}]\right)\right]^{2} \\
    &= {\bm{W}^{(l)}}^{2} \mathbb{E}_{\hat{\bm{h}}^{(l)}}\left[\hat{\bm{h}}^{(l)}-\mathbb{E}_{\hat{\bm{h}}^{(l)}}[\hat{\bm{h}}^{(l)}]\right]^{2}\\
    &= \bm{W}^{(l)} {\bm{W}^{(l)}}^\top \operatorname{cov}(\hat{\bm{h}}^{(l)})\\
    &=\sigma^{2}\bm{I}.
    \end{aligned}
\end{equation}

(2) Given that $\bm{W}^{(l)} {\bm{W}^{(l)}}^{\top}=\bm{I}$, and $\bm{W}^{(l)}$ is a  square orthogonal matrix, we have
\begin{equation}
\begin{split}
  \|\bm{H}^{(l)}\|_F 
  & = \sqrt{{\bm{H}^{(l)}}^\top \bm{H}^{(l)}} \\
  & =\sqrt{{\hat{\bm{H}^{(l)}}}^{\top} \bm{W}^{(l)} {\bm{W}^{(l)}}^{\top} \hat{\bm{H}}^{(l)}} \\
  & = \sqrt{{\hat{\bm{H}^{(l)}}}^{\top} \bm{I} \hat{\bm{H}}^{(l)}} \\
  & =\|\hat{\bm{H}}^{(l)}\|_F.  
\end{split}
\end{equation}

(3) We have a proof similar to proof of (2):

\begin{equation}
    \begin{aligned}
    \left\|\frac{\partial \mathcal{L}}{\partial \hat{\bm{h}}^{(l)}}\right\|_F
    & =\left\|\frac{\partial \mathcal{L}}{\partial \bm{h}^{(l)}} \bm{W}^{(l)}\right\| \\
    & =\sqrt{\frac{\partial \mathcal{L}}{\partial \bm{h}^{(l)}} \bm{W}^{(l)} {\bm{W}^{(l)}}^{\top} \frac{\partial \mathcal{L}^\top}{\partial \bm{h}^{(l)}}} \\
    & =\sqrt{\frac{\partial \mathcal{L}}{\partial \bm{h}^{(l)}} \bm{I} \frac{\partial \mathcal{L}^\top}{\partial \bm{h}^{(l)}}} \\
    & =\left\|\frac{\partial \mathcal{L}}{\partial \bm{h}^{(l)}}\right\|_F.
    \end{aligned}
\end{equation}

\end{proof}

\subsection{Appendix 3: The Statistics of Datasets}

The statistics of the datasets used in the node classification and graph classification tasks are summarized in Table~\ref{tab:dataset1} and Table~\ref{tab:dataset2}, respectively.

\begin{table}[h]
  \centering
  \small
  \caption{Summary of datasets used in our node classification task.}
  \label{tab:dataset1}
  \setlength{\tabcolsep}{1mm}{
  \begin{tabular}{ccrrr}
    \toprule
    Dataset &Classes &Nodes &Edges &Features\\
    \midrule
    Cora & 7 & 2,708& 5,429 &1,433\\
    Citeseer & 6 & 3,327& 4,732 &3,703\\
    Pubmed & 3& 19,717 &44,338 &500\\
    Cornell & 5 & 183 & 295 & 1,703 \\
    Texas & 5 & 183 & 309& 1,703\\
    Wisconsin & 5 &251 &499& 1,703\\
    ogbn-arxiv &40 &169, 343 & 1, 166, 243 & 128\\
    \bottomrule
  \end{tabular}}
\end{table}

\begin{table}[h]
  \centering
  \small
  \caption{Summary of datasets used in our graph classification task.}
  \label{tab:dataset2}
  \setlength{\tabcolsep}{1mm}{
  \begin{tabular}{ccrrr}
    \toprule
    Dataset &Graphs &Nodes (max) &Nodes (avg) & Classes\\  
    \midrule
    D\&D & 1178 & 5748& 284.32 &2\\
    PROTEINS & 1113 & 620& 39.06 & 2\\
    \bottomrule
  \end{tabular}}
\end{table}

\subsection{Appendix 4: Implementations}
We implement Ortho-GConv and reproduce the baselines using Pytorch~\cite{paszke2017automatic} and Pytorch Geometric~\cite{fey2019fast}. To provide a rigorous and fair comparison between the different models on each dataset, we use the same dataset splits and training procedures. For the semi and full-supervised node classification and graph classification tasks, we calculate the average test accuracy of 10, 5 and 10 runs, respectively. We adopt the Adam optimizer for model training. For the sake of reproducibility, the seeds of random numbers are set to the same.
We perform random hyper-parameter searches, and report the case that achieves the best accuracy on the validation set of each benchmark. 
For Ortho-GConv, we tune the following hyper-parameters: $T \in\{1,2,3,4,5,6\}$, $\beta \in\{0.1,0.2,...,1\}$, $\lambda \in\{1e-4, 5e-4\}$.

\subsection{Appendix 5: Results of Semi-supervised Node Classification}

For the semi-supervised node classification task, we conduct experiments on 3 well-known benchmark datasets: Cora~\cite{sen2008collective}, CiteSeer~\cite{sen2008collective} and PubMed~\cite{sen2008collective}. 

We conduct two experiments to evaluate the performance of Ortho-GConv for semi-supervised node classification task, namely the comparison with different layers and comparison with SOTA. We follow the standard experimental setting proposed by \cite{split}. 

\begin{table*}[htp]
  \centering
  \small
  \caption{Testing mean accuracy (\%) comparisons on different backbones with/without Ortho-GConv on semi-supervised node classification tasks.}
  \label{tab:semi_overall_backbone}
    \centering
    \begin{tabular}{ll|cccccc}
    \toprule
    \multirow{2}{*}{Dataset}  & \multirow{2}{*}{Backbone} & \multicolumn{2}{c}{2 layers} & \multicolumn{2}{c}{4 layers} & \multicolumn{2}{c}{8 layers} \\
      && \multicolumn{1}{c}{original~} & \multicolumn{1}{c}{Ortho-GConv} & \multicolumn{1}{c}{original~} & \multicolumn{1}{c}{Ortho-GConv} & \multicolumn{1}{c}{original~} & \multicolumn{1}{c}{Ortho-GConv}  \\
    \hline
    \multirow{3}{*}{Cora} & GCN& 81.52$\pm$0.71  & {\bf83.65$\pm$0.52}& 78.12$\pm$1.67& {\bf82.66$\pm$0.68}  & 63.11$\pm$7.47& {\bf79.08$\pm$0.63}\\
    & JKNet & 81.87$\pm$0.75& {\bf82.32$\pm$0.37 }& 80.59$\pm$1.13& {\bf82.06$\pm$0.34}&79.50$\pm$1.29  & {\bf81.98$\pm$0.31}\\
      & GCNII & 82.20$\pm$0.60  &{\bf83.63$\pm$0.56}  & 83.00$\pm$0.65& {\bf84.00$\pm$0.56}& 84.60$\pm$0.40  & {\bf84.62$\pm$0.37}  \\
    \hline 
    \multirow{3}{*}{Citeseer} & GCN& 71.15$\pm$1.49  & {\bf71.53$\pm$0.61}& 62.15$\pm$1.68& {\bf71.24$\pm$0.70}  & 53.76$\pm$5.45& {\bf68.59$\pm$1.74}  \\
      & JKNet & 70.38$\pm$1.22&{\bf 70.81$\pm$ 0.46}  & 70.47$\pm$0.45&{\bf70.78$\pm$0.74} & 70.30$\pm$1.09  &{\bf70.87$\pm$0.35}  \\
    & GCNII &68.20$\pm$0.85& {\bf71.65$\pm$0.46}  & 68.90$\pm$0.64& {\bf72.36$\pm$0.50} & 72.10$\pm$0.88  & {\bf 72.58$\pm$0.51}\\
    \hline
    \multirow{3}{*}{Pubmed}& GCN& 79.02$\pm$5.40  & {\bf 79.78$\pm$0.38} & 54.92$\pm$5.42 & {\bf79.78$\pm$0.42}  &  51.33$\pm$6.84 &{\bf 78.32$\pm$0.65}\\
      & JKNet &78.62$\pm$0.18&{\bf79.22$\pm$ 0.35}  & 77.85$\pm$1.28& {\bf79.10$\pm$0.27} &75.27$\pm$1.29  & {\bf79.78$\pm$0.35}\\
    & GCNII &78.20$\pm$0.81& {\bf79.17 $\pm$0.18} & 78.10$\pm$0.90&{\bf79.45$\pm$0.29}  & 78.70$\pm$1.03  &{\bf 79.89$\pm$0.38} \\
    \hline
    \bottomrule
    \end{tabular}
\end{table*}

\begin{table*}[h]
  \centering
  \small
  \caption{Accuracy (\%) comparisons with SOTA on full-supervised tasks. The highest performances are in bold. The number in parentheses denotes the number of layers.}
  \label{tab:full_sota2}
    \begin{tabular}{ccccccc}
    \toprule
    Method&  Cora  &Cite.  &Pumb.&Corn.  &Texa.  &Wisc.\cr
    \midrule
    GCN&85.77  &73.68  &87.91  &52.70  &55.14  &49.02  \cr
    GAT &86.37  &74.32  &87.62  &54.32  &58.38  &49.41 \cr
    Geom-GCN&85.19  &{\bf77.99}  &90.05&56.76  &57.58  &58.24  \cr
    APPNP &87.87  &76.53  &89.40  &73.51  &65.41  &69.02 \cr
    JKNet&85.25(16)  &75.85(8)  &88.94(64)&57.30(4)  &56.49(32) &48.82(8)\cr
    Incep(Drop)&86.86(8)&76.83(64)&89.18(4)&61.62(16)&57.84(8)&50.20(8)  \cr
    GCNII&88.49(64)&77.08(64)&89.78(8)&74.86(16)&71.46(16)&75.30(8)  \cr
    Ortho-GCNII&{\bf 88.81}(4)&77.26(2)&{\bf90.30(2)}&{\bf76.22}(2)&{\bf77.84}(2)&{\bf77.25}(4)  \cr
    \bottomrule
    \end{tabular}
\end{table*}

\noindent \textbf{(1) Comparison with Different Layers.} We adopt three backbones - GCN~\cite{kipf2016semi}, JKNet~\cite{xu2018representation} and GCNII\cite{chen2020simple} on Cora, Citeseer and Pubmed, with depths ranging from 2 to 8. 
For all the models, we train them until they undergo convergence and then report the best test results.
The final results of Ortho-GConv at the 2/4/8 layers compared with the original models are reported in Table~\ref{tab:semi_overall_backbone}, which takes accuracy as an evaluation metric.
After analyzing the results, we make the following conclusions:
\begin{table}[h]
\small
\centering
  \caption{Accuracy (\%) comparisons with SOTA on semi-supervised tasks. The highest performances are in bold. The number in parentheses denotes the number of layers.}
  \label{tab:semi_sota}
    \setlength{\tabcolsep}{1mm}{
    \begin{tabular}{cccc}
    \toprule
    Method&  Cora  &Citeseer  &Pubmed  \cr
    \midrule
    GCN&81.5  &71.1  &79.0  \cr
    GAT &83.1  &70.8  &78.5\cr
    APPNP &83.3  &71.8  &80.1  \cr
    JKNet&81.9(2)  &69.8(16)  &78.1(32)  \cr
    JKNet(Drop)&83.3(4)  &72.6(16)  &79.2(32)  \cr
    Incep(Drop)&83.5(64)&72.7(4)&79.5(4)\cr
    SGC &81.7&71.7&78.9\cr
    DAGNN &84.4&73.3&{\bf80.5}\cr
    GCNII&85.5$\pm$ 0.5 (64)&73.4$\pm$0.6(32)&80.2$\pm$0.4(32)\cr
    Ortho-GCNII &{\bf85.6$\pm$ 0.4}(64)&{\bf74.2$\pm$ 0.4}(32)&80.1$\pm$0.3 (16)\cr
    \bottomrule
    \end{tabular}}
\end{table}

\romannumeral1) Ortho-GConv achieves better performance than the original models. It suggests that the accuracy can be improved by applying our proposed Ortho-GConv.
For example, Ortho-GConv obtains average improvements of 1.6\% over the backbones with 2 layers on Cora , while achieving considerable increase of 9.3\% in the case of 8 layers.

\romannumeral2) 
As the number of layers increases, the performances of the three backbones without applying Ortho-GConv decrease significantly, while the backbones equip with Ortho-GConv consistently maintain higher experimental performances. This demonstrates that the orthogonal weight constraint could tackle the performance degrading problem of shallow GNNs, through stabilize the forward inference and backward gradient propagation as analyzed in the previous section.

\noindent \textbf{(2) Comparison with SOTA.} In order to verify the overall performance of Ortho-GConv, we choose the best performance from each backbone with Ortho-GConv and compare with the current popular SOTA approaches, including three shallow models (GCN~\cite{kipf2016semi}, GAT~\cite{GAT} and APPNP~\cite{APPNP}), DropEdge \cite{rong2019dropedge}, SGC \cite{wu2019simplifying} and DAGNN \cite{liu2020towards}. 
The comparison results are provided in Table~\ref{tab:semi_sota}.
After analyzing the results, we make the following conclusions:

Clearly, the experimental results demonstrate that Ortho-GConv obtains an effective enhancement against SOTA methods. 
We can observe that Ortho-GConv gains 0.2\% average improvement on all of datasets. Particularly, Ortho-GConv delivers 1.1\% improvement over GCNII on Citeseer, which indicates that Ortho-GConv has a remarkable boost considering the three perspectives of Ortho-GConv on GNN models.

\subsection{Appendix 6: Results of Full-supervised Node Classification}
Table~\ref{tab:full_sota2} summarizes the highest performances of the different approaches by screening the model depths. It is observed that Ortho-GCNII, a model implementing our Ortho-GConv over GCNII backbone, generally delivers the competitive accuracy on all the considered datasets. This validate the effectiveness of Ortho-GConv to further boost model performance, which could be applied to explore the powerful GNNs in the future.

\end{document}